\documentclass[letterpaper, 10 pt, conference]{IEEEtran} 
\IEEEoverridecommandlockouts            

\usepackage{colortbl}
\usepackage[table]{xcolor}
\usepackage{amsfonts}
\usepackage{subfigure}
\usepackage[justification=centering]{caption}
\usepackage{mdwlist}
\usepackage{cite}
\usepackage{amsfonts}
\usepackage{amssymb}
\usepackage{amsmath}
\usepackage{color}
\usepackage[linesnumbered,boxed,ruled, commentsnumbered]{algorithm2e}
\usepackage{multirow}
\usepackage{bm}
\usepackage{booktabs}
\usepackage{subfigure}
\usepackage{graphics}
\usepackage{graphicx}
\usepackage{ifpdf}
\usepackage{epstopdf}
\usepackage{mathrsfs}
\usepackage{amsmath}
\usepackage{amsthm}
\usepackage{float}
\usepackage{epsfig}
\usepackage{multirow}
\usepackage{enumerate}
\usepackage{nomencl}
\usepackage{threeparttable}
\usepackage{arydshln}
\usepackage{verbatim}
\usepackage{tikz,times}
\usepackage{algorithmic}
\usetikzlibrary{calc, shapes, backgrounds}
\newtheorem{definition}{Definition}
\newtheorem{theorem}{Theorem}
\newtheorem{lemma}{Lemma}
\newtheorem{remark}{Remark}
\newtheorem{proposition}{Proposition}

\definecolor{mygray}{gray}{.9}
\definecolor{mypink}{rgb}{.99,.91,.95}
\definecolor{mycyan}{cmyk}{.3,0,0,0}

\makeatletter
\newcommand\figcaption{\def\@captype{figure}\caption}
\makeatother

\makeatletter
\renewcommand{\maketag@@@}[1]{\hbox{\m@th\normalsize\normalfont#1}}%
\makeatother

\begin{document}
\title{\LARGE \bf Speed Planning Using B\'{e}zier Polynomials with Trapezoidal Corridors}
\author{Jialun Li\textsuperscript{1}, Xiaojia Xie\textsuperscript{2}, Hengbo Ma\textsuperscript{3}, 
Xiao Liu\textsuperscript{2} and Jianping He\textsuperscript{1}
\thanks{The authors\textsuperscript{1} are with the Department of Automation, Shanghai Jiao Tong University, and Key Laboratory of System Control and Information Processing, Ministry of Education of China, Shanghai 200240, China. E-mail: \{jialunli, jphe\}@sjtu.edu.cn. The authors\textsuperscript{2} are with Megvii (Face++) Technology Inc., Beijing, China. E-mail: \{xiexiaojia, liuxiao\}@megvii.com. 
The author\textsuperscript{3} is with University of California, Berkeley, CA 94720, USA. E-mail: hengbo\_ma@berkeley.edu.}}
\captionsetup{font={small}}
\maketitle

\begin{abstract}
To generate safe and real-time trajectories for an autonomous vehicle in dynamic environments, path and speed decoupled planning methods are often considered. This paper studies speed planning, which mainly deals with dynamic obstacle avoidance given the planning path. The main challenges lie in the decisions in non-convex space and the trade-off between safety, comfort and efficiency performances. This work uses dynamic programming to search heuristic waypoints on the S-T graph and to construct convex feasible spaces. Further, a piecewise B\'{e}zier polynomials optimization approach with trapezoidal corridors is presented, which theoretically guarantees the safety and optimality of the trajectory. The simulations verify the effectiveness of the proposed approach. 
\end{abstract}
\IEEEpeerreviewmaketitle

\section{Introduction}
Autonomous vehicles are promising to revolutionize transportation systems and change the ways how people travel \cite{schwarting2018planning,paden2016survey}. To interact with other agents on road, the self-driving car adjusts its path and speed over time constantly based on perception information. This task can be formulated as a problem to optimizes a high-dimensional trajectory in terms of comfort and energy saving while satisfying safety and dynamic feasibility constraints. However, solving the original optimization problem is intractable with limited computation time and infeasible for online running. Therefore, it is challenging for the ego vehicle to generate collision-free, real-time and comfortable trajectories to react to other road participants. 

To deal with this issue, there are two major trajectory generation frameworks, spatio-temporal planning  \cite{ding2019safe,ziegler2014trajectory,mercy2016real} and path/speed decoupled planning in Fren\'{e}t frame \cite{gu2016runtime,xu2012real,li2015real}. These two approaches share the same hierarchical ideas to achieve real-time planning by searching heuristic solutions first and optimizing the preliminary results in convex subspaces later. 

The spatio-temporal planning considers the spatial and temporal maneuvers simultaneously. The search and optimization processes are completed in three-dimensional spaces. Correspondingly, the decoupled method decomposes the 3D planning into two stages as path planning and speed planning. In each planning cycle, path planning is executed to generate a path to avoid static, oncoming and low-speed vehicles. Then, speed planning adjusts the vehicle's speed to keep a safe distance from dynamic obstacles which block the formed path. Although, the layered approach is prone to be suboptimal with the appearance of dynamic obstacles compared to the 3D optimization, it is more flexible for the separated design process. Besides, the computational complexity can be reduced by this approach as discussed in \cite{gu2016runtime,xu2012real,li2015real}. Hence, the decoupled planning framework is widely adopted by academia and industry for its flexibility and efficiency. 
\begin{figure}[t]
\begin{center}
\subfigure[]{\label{fig1.a}
\includegraphics[height=4cm]{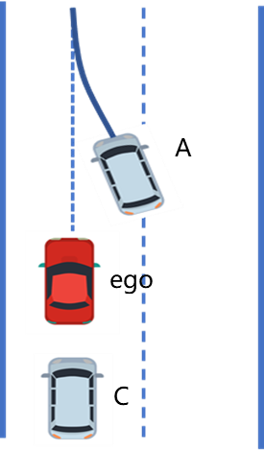}}
\subfigure[]
{\label{fig1.b}
\includegraphics[height=3.6cm]{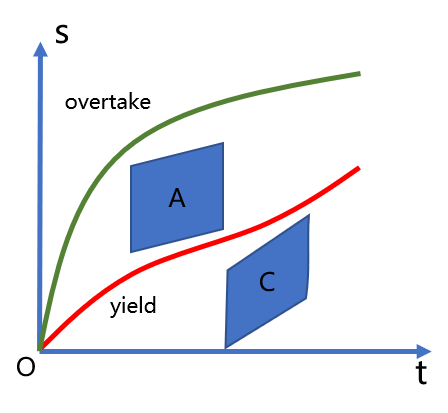}}
\caption{A merge scenario and its S-T graph.}
\label{fig1}
\end{center}
\vspace*{-0.8cm}
\end{figure}

In the hierarchical planning framework, speed planning plays an important role in avoiding other dynamic vehicles. For example, when another vehicle merges into the same lane, the ego vehicle may follow the lane-changing vehicle while maintaining a distance from the car behind, as shown in Fig.\ref{fig1.a}. During this process, the speed of the ego vehicle is expected to slow down first and then accelerate smoothly in a short response time. The speed profiles determine the distances from obstacles, forward and lateral accelerations, which influence the safety and comfort of passengers. 

For the speed planning, \cite{lipp2014minimum} and \cite{zhang2018speed} present a convex optimization method with friction circle as the constraint. Liu \cite{liu2017speed} presents a novel slack convex feasible set algorithm to optimize the time stamp of each waypoint with fixed stations. These approaches ensure the safety and comfort. For better computational efficiency, the first search for heuristic profiles and post optimization paradigm is mostly taken. The search methods include A* \cite{hubmann2016generic}, RRT \cite{du2020speed} and dynamic programming \cite{fan2018baidu}. As for the optimization, there are two main approaches. One is to optimize the stations at discretized time instants directly  \cite{zhang2019optimal,meng2019decoupled}. Another is to parameterize the speed profile with polynomials and convert the speed optimization problem into finding optimal polynomial coefficients which satisfy certain constraints \cite{fan2018baidu}.

In all the above speed optimization methods, the safety constraints are imposed at discretized time instant. The problem lies in that the safety limitations should hold for the whole planning horizon and the safety between two adjacent sampling time stamps cannot be guaranteed. These works tackle this problem by shortening the time interval to impose the constraints. However, this will lead to more decision variables and much more computation time. 

Motivated by the above observations, in this paper, we focus on safety enforcement for the whole speed planning period in dynamic environments with feasibility, efficiency and comfort considerations. We use piecewise B\'{e}zier polynomials to optimize the speed profile on the S-T graph. Specially, we provide the sufficient condition on coefficients of the B\'{e}zier polynomial to enforce the station curve in a serial of trapezoidal corridors\footnote{Corridors imply subspaces of safe regions where the B\'{e}zier curves always lie in (Refer to Section III.C for details). This term is firstly presented in the safe planning of unmanned aerial vehicles (UAVs).} for the first time, by the convex hull property of the B\'{e}zier polynomial. During the search phase, dynamic programming is adopted. The main contributions of our work are summarized as follows. 
\begin{itemize}
\item We use dynamic programming (DP) to search feasible space and construct convex regions from the non-convex ones for online feasibility. The piecewise heuristic polynomials of stations with time stamps are achieved at the same time to avoid minimizing the distance between station curve and discretized heuristic points for the concise objective function and efficiency.
\item We provide the sufficient condition on coefficients of the B\'{e}zier polynomial to guarantee the station curve in trapezoidal corridors theoretically. Compared with existing condition of enforcement in the piecewise rectangular corridors \cite{ding2019safe}, the condition is relaxed and solution space is demonstrably enlarged, which leads to a lower cost function and better comfort. 
\item We conduct simulations to verify the effectiveness of the proposed method. Our method is better than the classical B\'{e}zier polynomial method in terms of the comfort with lower failure rates in highly dynamic environments. The proposed method also outperforms other works on efficiency, e.g., speed planning modules of the public road planner \cite{zhang2019optimal} and the EM motion planner \cite{fan2018baidu}. We release our module as an open-source package.
\end{itemize}

\section{Related Works}
\textbf{Speed Planning} Speed profiles can be generated by means of
(\romannumeral1) searching and optimizing. (\romannumeral2) sampling lattices and selecting by cost. (\romannumeral3) approximated optimization. The first method searches for the best candidate speed profile and optimizes the curve for smoothness. This method is commonly adopted for optimality and efficiency. Xu first presents a method of selecting the best lattice and conducting post optimization \cite{xu2012real}. The Baidu EM motion planner uses dynamic programming for search and piecewise monomial polynomials for optimization \cite{fan2018baidu}. In \cite{zhou2020autonomous}, a trajectory planning algorithm for open spaces is presented. The heuristic speed points are calculated based on the vehicle dynamic and directly optimized with Piecewise-Jerk Speed Optimization. As for the approach (\romannumeral2), different speed lattices are sampled and combined with path lattices. The generated local spatial-temporal trajectories are evaluated and the best one is selected. The core problem is how to construct the lattices to make up the suboptimality of discretization. Related works see \cite{gu2016runtime,li2015real}.
Besides, there are some works directly optimizing the speed profile. Liu presents a novel slack convex feasible set algorithm \cite{liu2017speed} and Qian proposes a MPC method \cite{qian2016motion}. 

\textbf{B\'{e}zier Polynomials Based Planning}
B\'{e}zier polynomial is the linear combination of Bernstein basis and can be obtained by linear transformation from polynomial with monomial basis. B\'{e}zier polynomial has useful properties for the path generation and speed planning problem. In \cite{gonzalez2016speed}, a smooth and continuous speed profile is computed by proper curve concatenation without optimization and considering dynamic obstacles. In the trajectory generation of UAVs, the B\'{e}zier polynomial is widely taken with rectangular corridors in 2D space to describe safe areas from obstacles \cite{gao2018online}, \cite{zhou2019robust}. As for autonomous driving, Ding borrows this idea and proposes a way of generating sequential rectangular corridors where B\'{e}zier polynomials for the speed curve are optimized \cite{ding2019safe}.

Corridors based B\'{e}zier polynomial optimization is prospective to enhance safety, comfort and efficiency in speed planning. 
However, the shapes of safe regions on S-T graph are quite different from those for UAVs to generate trajectories. Then, the shapes and generations of  corridors of two scenes are distinct from each other. This will lead to different limitations on coefficients and optimization performances. For the S-T graph, the boundaries of obstacles are straight lines or parabolic curves, since the accelerations of obstacles are usually assumed to be constant in the planning horizon. In the UAV navigation scenarios, the obstacles are  circle, rectangles and polygons where rectangular corridors are easy to be generated by cube inflations. Hence, the following parts aim to solve the gap between profile generation on S-T graph and existing B\'{e}zier polynomial optimization methods. 
\section{S-T Graph and Trajectory Formulation}
\subsection{S-T Graph and Non-convex Optimizations}
As mentioned above, speed planning is to append a speed component to the path. In this process, the longitudinal and lateral comfort, physical and traffic constraints and safety constraints should be considered. The S-T graph is an approach to analyze this issue intuitively (Fig.\ref{fig1.b}). 

Before executing the motion planning module, the ego vehicle predicts trajectories of surrounding vehicles. If its path is blocked by other vehicles in the planning horizon, the stations of these obstacles over time will be mapped to the S-T graph. We take the common case that they move with constant speeds for simplicity. Then, the area denoting each obstacle on the S-T graph is 
a parallelogram. The side length of parallelogram along s axis equals to the length of the vehicle plus half length of ego vehicle for convenience of safety regions representation. 


A speed profile of the ego vehicle on the S-T graph reflects its distances from others with time and its decisions on how to avoid these obstacles, such as yielding, overtaking, following and so on. The curves below blocked areas mean to yield the car in front, while those above mean to overtake or keep distance from the car behind. To keep safe, the feasible space of the speed curve should not overlap with the regions projected by obstacles.

As shown in Fig.\ref{fig2.a}, the constraints on stations are imposed at discrete time stamps for safety. In some time slots, the allowed stations are split into two disconnected intervals by obstacles. When the station-time curve is optimized with these segmented safety constraints, the problem is non-convex and hard to be solved in milliseconds for online running. \vspace{-0.4cm}
\subsection{B\'{e}zier Polynomial and Its Properties}
The B\'{e}zier polynomial is the polynomial function represented by linear combinations of Bernstein basis. The $n$-th ordered B\'{e}zier polynomial is written as
\vspace{-0.2cm}
\begin{equation}\label{Bezier_func}
\begin{aligned}
B(t)= c_0b_{n}^{0}(t)+c_1b_{n}^{1}(t)+\dots+c_{n}b_{n}^{n}(t)=\sum_{i=0}^{n}c_ib_{n}^{i}(t),
\end{aligned}
\end{equation}
where the Bernstein basis satisfies $b_{n}^{i}(t)=C_n^{i} \cdot t^{i} \cdot (1-t)^{n-i},t\in[0,1]$. The coefficients of the polynomial $c_i\ (i = 0,1,\dots,n)$ are also called control points. Compared to the monomial polynomial, B\'{e}zier curve has following properties. 

\noindent i) The time interval is defined on $t\in [0,1]$. 

\noindent ii) The B\'{e}zier polynomial starts at control point $B(0)=c_0$ and ends at $B(1)=c_n$. 

\noindent iii) The B\'{e}zier curve $B(t)$ is confined within the convex hull of control points. It can be described by the proposition \ref{pro1}. 
\begin{proposition}\label{pro1}
For B\'{e}zier polynomial $B(t)=\sum_{i=0}^{n}c_ib_{n}^{i}(t)$, if the coefficients satisfy
$\underline {p_0}\leq c_i \leq \overline {p_0},\ i =0,1,\dots,n$, 
we have $\underline {p_0} \leq B(t) \leq \overline {p_0},t\in[0,1]$.
\end{proposition}
This is widely used in the B\'{e}zier curve optimization with rectangular corridor. 
We give a brief proof in the appendix. \\
\noindent iv) The derivative of $B(t)$, $\dot{B}(t)$, can also be written as a B\'{e}zier polynomial with control points $c_{i}^{1}=n\cdot(c_{i+1}-c_{i})$,\;$i=0,1,\dots,n-1$. By this way, we are also able to calculate arbitrary derivatives of $B(t)$. Similarly, control points of $\frac{\mathrm{d}^{l+1}B(t)}{\mathrm{d}t^{l+1}}$ 
and $\frac{\mathrm{d}^{l}B(t)}{\mathrm{d}t^{l}}$ satisfy $c_{i}^{l+1}=(n-l)(c_{i+1}^{l}-c_{i}^{l})$.
\subsection{Speed Curve Formulation by Piecewise B\'{e}zier Polynomials}
Since the station-time points obtained by heuristic searching are lack of comfort and smoothness considerations, it requires a polynomial approach for optimization subsequently. Instead of using traditional polynomials with monomial basis, we choose B\'{e}zier polynomials to ensure the speed profile in the safe region theoretically. The safety property guaranteed by B\'{e}zier function is that constraints on control points can limit the value of polynomials during the whole time period. This uses the convex hull property (iii) above and proposition \ref{pro1}. 

The piecewise polynomials are adopted to replace higher order polynomial. This will guarantee fitting performance while avoiding numerical instability. Since $B(t)$ is defined on a fixed time interval $[0,1]$, for a piece of trajectory during $[T_k,T_{k+1}]$, we use translation and scaling methods same as \cite{gao2018online} to transform the time domain into correct time unit.
Then the station with time can be represented as 
\begin{equation}
\centering
s(t)=\left\{
\begin{array}{c}
   h_{0}B_0\left(\frac{t-T_0}{h_0}\right),t\in[0,T_1]\\  
   h_{1}B_1\left(\frac{t-T_1}{h_1}\right),t\in[T_1,T_2]\\
          \vdots  \\
   h_{m}B_m\left(\frac{t-T_m}{h_m}\right),t\in[T_m,T_{m+1}]. \\
\end{array}
\right.
\end{equation}
\section{Heuristic Searching and corridor generation}
\subsection{Dynamic Programming}
Since the speed optimization problem is non-convex, it is infeasible to be solved directly during online planning. Hence, we take the search method for heuristic profiles and post convex optimization to generate a smooth speed curve. We first discretize the S-T graph into grids. The grids have the same time interval as $\Delta t_1$. Then we use dynamic programming to search on the nodes, vertices of grids, to obtain the heuristic stations with discrete time.

In dynamic programming algorithm, the states are the nodes of S-T graph as $(t_i,s_j)$. Let the cost function $\mathrm{cost}(t_i,s_j)$ denote the total cost from start state $(0,s_0)$ to $(t_i,s_j)$.
This cost can be achieved by the sum of cost at consecutive single node selected and the cost of state transitions $\mathrm{cost}_{edge}((t_{i-1},s_k),(t_i,s_j))$. The iterative equation between parent node $(t_{i-1},s_k)$ and child node $(t_i,s_j)$ is given by 
\begin{equation}
\begin{aligned}
\mathrm{cost}(t_i,s_j)&=\min_{t_i,s_j}\{\mathrm{cost}(t_i,s_j),\mathrm{cost}(t_{i-1},s_k)\\&+\mathrm{cost}_{node}(t_i,s_j)
+\mathrm{cost}_{edge}((t_{i-1},s_k),(t_i,s_j))\}.
\end{aligned}
\end{equation}
The term $\mathrm{cost}_{node}(t_{i},s_j)$ is the cost at single node $(t_i,s_j)$, determined by 
distances from obstacles and expected terminal station at $T$. 
The transitions between states $(t_{i-1},s_k)$ and $(t_i,s_j)$ satisfy limited velocity and acceleration constraints. After calculations of costs by iterations, the optimal end point is chosen as $(t_i,s_j)$ with $t_i=T$ and lowest cost. Then, the heuristic points $(t^r_j,s^r_j),\;j=0,1,\dots,j_{m+1}$ are obtained by visiting the parent of each node backward from the end to the start. The piecewise heuristic station curves are represented by linear interpolation as 
\begin{equation}
\begin{aligned}
s^r_j(t)=\frac{s_{j+1}^r-s_{j}^r}{\Delta t_1}(t-t_j)+s_j^r,\;j=0,1,\dots,j_{m+1}.
\end{aligned}
\end{equation}

These heuristic stations by DP not only provide a reference and warm start for the next optimization process, but also imply decisions concerning the obstacles. Further, we can follow their physical meanings and discard another region induced by the obstacle as shown in Fig.\ref{fig2.b}. Then, the constraint on station at each time instant is a continuous interval. As a result, the non-convex space can be converted into piecewise convex ones for next optimization. 
\begin{figure}[t]
\begin{center}
\subfigure[Nonconvex safe regions.]{\label{fig2.a}
\includegraphics[width=4cm,height=2.6cm]{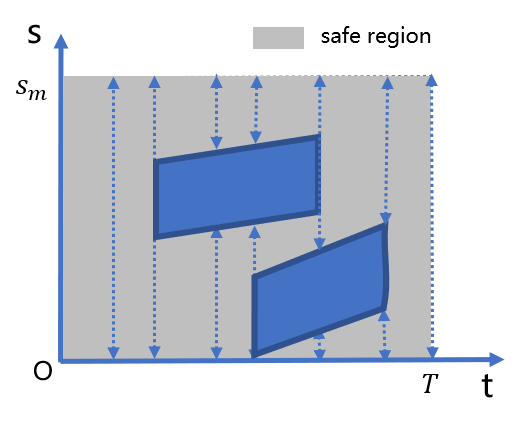}}
\subfigure[DP search and piecewise convex safe regions.]{\label{fig2.b}
\includegraphics[width=4cm,height=2.6cm]{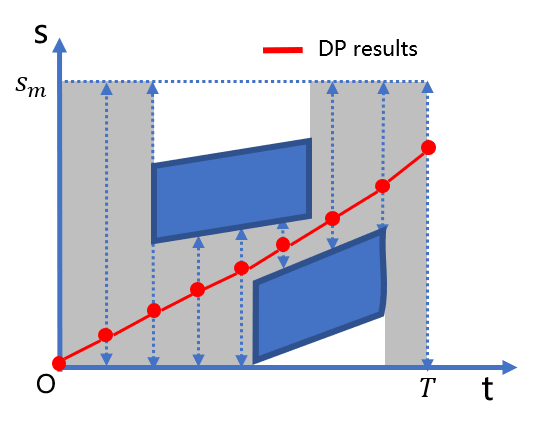}}
\subfigure[Representations of safe regions.]{\label{fig2.c}
\includegraphics[width=4cm,height=2.6cm]{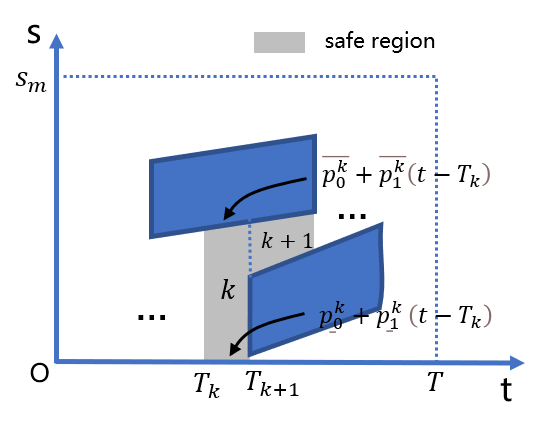}}
\subfigure[Generations of safe regions.]{\label{fig2.d}
\includegraphics[width=4cm,height=2.6cm]{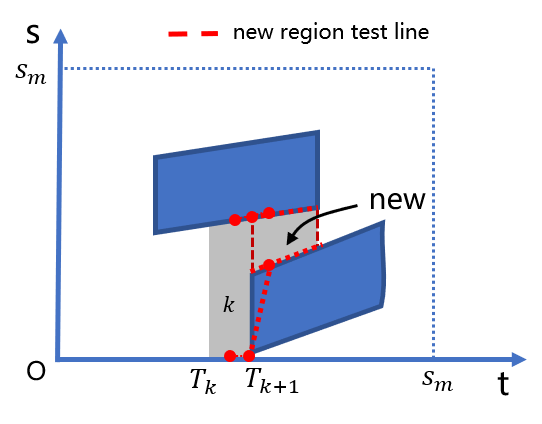}}
\caption{Safety Regions and their Generations on S-T graph.}
\label{fig2}
\end{center}
\vspace*{-0.8cm}
\end{figure}
\subsection{Piecewise Convex Safe Regions Representations}
Suppose the whole safe region induced by DP are divide into $m+1$ pieces with time intervals $[T_0,T_1],\dots,[T_m,T]$ and $T=T_{m+1}$. As shown in Fig.\ref{fig2.c}, the $k$-th convex safe regions can be represented as
\begin{equation}
\begin{aligned}
\mathcal{S}_k=\{(t_i,s_i)|\underline{p_0^{k}}+h_k\underline {p_1^{k}}\frac{t_i-T_k}{h_k}\leq s_i \leq \overline {p_0^{k}}+h_k\overline {p_1^{k}}\frac{t_i-T_k}{h_k},\\t_i\in[T_k,T_{k+1}]\},
\end{aligned}
\end{equation}
where $\underline{p_0^{k}},\underline{p_1^{k}}$ are bias and skew of the lower bound and $\overline{p_0^{k}},\overline{p_1^{k}}$ are those of the upper bound. $h_k$ denotes the length of $k$-th time interval and satisfies $h_k = T_{k+1}-T_k$, $k=0,1,\dots,m$. 

Then, the whole safe region is 
$$\mathcal{S}=\mathcal{S}_0\cup\dots\cup\mathcal{S}_m.$$ The speed planning is safe if 
$\ \forall t_0 \in [0,T],s(t_0)\in \mathcal{S}$, which is equivalent to 
for $t_0\in[T_k,T_{k+1}]$, $s(t_0)\in\mathcal{S}_k$, $k=0,1,\dots,m$, i.e.
\begin{equation}\label{safety1}
\underline{p_0^{k}}+h_k\underline {p_1^{k}}\frac{t_0-T_k}{h_k}\leq s(t_0) \leq \overline {p_0^{k}}+h_k\overline {p_1^{k}}\frac{t_0-T_k}{h_k}.
\end{equation}
\subsection{Piecewise Convex Safe Regions Generations}
As discussed above, the constraint \eqref{safety1} is expected to be followed for safety in whole planning horizon. However, the bounds on stations are stored as lower bounds $lb$ and upper bounds $ub$ at discrete time stamps after DP processing. The values of $\underline{p_0^{k}}$, $\underline{p_1^{k}}$, $\overline{p_0^{k}}$, $\overline{p_1^{k}}$ and $h_k$ in \eqref{safety1} cannot be achieved directly. Hence, the algorithm \ref{algo-1} is presented to obtain these values and to generate sequential convex safe regions.

Fig.\ref{fig2.d} illustrates the generation of piecewise safe convex regions by algorithm \ref{algo-1}. In the algorithm, the heuristic positions are first sampled with the equal time interval $\Delta t_2$ as a serial of boundary points with $\Delta t_2=\frac{\Delta t_1}{N}$, 
$N\in \mathbb{N}_+$.
Then, three adjacent points in upper and lower boundaries are detected respectively to see whether they are on the same line. If they are not collinear, a new region will be created with different skews. After the generations of regions, the function RegionSplit will check the length of each corridor. If there is some region's length above $1s$, the region will be split into multiple regions with time intervals not exceeding $1s$. This operation aims to avoid underfitting of polynomial optimization and guarantee optimization performance. Accordingly, the function RegionMerge aims to merge the regions with small time intervals into adjacent corridors to avoid overfitting and speed up the optimization process.

It should be noted that the creations of regions are dependent on the descriptions of predicted interacting vehicles. For simplicity, we assume that the obstacles are traveling at constant velocities. Hence, if the boundaries of obstacles are not linear, they should be processed to linear bounds first.
\begin{algorithm}[t]
 \small
    \caption{Piecewise Convex Regions Generation}
    \label{algo-1}
    \begin{algorithmic}[1]
    \REQUIRE~~{$\ \  lb,ub,nums,\Delta t_2$}
    \ENSURE~~{$regions$}
    \STATE \textbf{Initialize}: $regions[0], i=0, j = 1$\
    \STATE ConstructRegion($reg,0,lb[0],ub[0],lb[1],ub[1]$)
    \STATE Insert($regions,reg$)
    \FOR {$i\gets 2$\ \textbf{to} $nums-2$}
    { \STATE $lskew = (lb[i]-lb[i-1])/\Delta t_2$
      \STATE $uskew = (ub[i]-ub[i-1])/\Delta t_2$
      \IF {$\|lskew-regions[j-1].lskew\|>\varepsilon$ or 
    $\|uskew-regions[j-1].uskew\|>\varepsilon$}
    { 
    	\STATE $regions[j - 1].t_{end} = i$
    	\STATE $regions[j - 1].t = (regions[j - 1].t_{end}-regions[j - 1].t_{beg})*\Delta t_2$
    	\STATE ConstructRegion($reg,j,lb[i],ub[i],lb[i+1],ub[i+1]$)
    	\STATE Insert($regions,reg$)
    	\STATE $j\gets j+1 $
    }
    \ENDIF
    }
    \ENDFOR
    \STATE $regions[j-1].t_{end}=n-1$
    \STATE $regions[j - 1].t = (regions[j - 1].t_{end}-regions[j - 1].t_{beg})*\Delta t_2$
    \STATE RegionSplit($regions$)
    \STATE RegionMerge($regions$)
    \STATE \textbf{Return} $regions$
    \end{algorithmic}
\end{algorithm}
\begin{algorithm}[t]
 \small
    \caption{ConstructRegion($reg,j,lb[i],ub[i],lb[i+1],ub[i+1]$)}
    \label{algo-2}
    \begin{algorithmic}[1]
    \REQUIRE~~{$\ \  reg,j,lb[i],ub[i]$}
    \ENSURE~~{$corridors$}
    \STATE $reg.t_{beg} = i$
    \STATE $reg.lskew = (lb[i+1]-lb[i])/\Delta t_2$
    \STATE $reg.lbias = lb[i]$
    \STATE $reg.uskew = (ub[i+1]-ub[i])/\Delta t_2$
    \STATE $reg.ubias = ub[i]$
    \STATE Insert($regions,reg$)
    \end{algorithmic}
\end{algorithm}

\section{Piecewise B\'{e}zier Polynomial Optimization}
\subsection{Problem of Safety Enforcement in Rectangular Corridors}
To enforce the station curve in the safe region $\mathcal{S}$, the convex hull property of the B\'{e}zier function is used and the definition of corridor is given first.
\begin{definition}
Let coefficients of the B\'{e}zier Polynomial satisfy 
$c_i\in\Omega,i=0,1,\dots,n$ and then the Bezier curve $B(t)$ resides in a set $\mathcal{S}^{cor}$ for the convex hull property. If $\mathcal{S}^{cor}$ is the subset of safe set $\mathcal{S}$, i.e. $\mathcal{S}^{cor} \subseteq \mathcal{S}$, $\mathcal{S}^{cor}$ is called a corridor.
\end{definition}
With the concept of corridor, the safety can be checked by calculation of $\mathcal{S}^{cor}$ by $c_i\in\Omega$ and contrast with $\mathcal{S}$. The shape of corridor is defined as that of region covered by $\mathcal{S}^{cor}$. 
Usually, a serial of rectangular corridors are chosen to gurantee the safety of trajectory parameterized by coefficients of B\'{e}zier polynomials. 
Constraints on the coefficients are described by the following proposition.
\begin{proposition}\label{propo2}
For arbitrary pieces of the trajectory, if there exists control points 
$c_i^{k}\in$ $\Omega_1^{k}=\{c^k|\underline {p^{k}_0}+h_k\underline {p^{k}_1} \leq c^k \leq \overline {p^{k}_0}, i=0,1,\dots,n, k=0,1,\dots,m\}$, $s(t)$ is safe on the S-T graph and $\mathcal{S}^{cor}$ is a rectangular corridor 
$\mathcal{S}^{rec}$. 
\end{proposition}
\begin{proof}
With conditions $c_i^{k}\in$ $\Omega_1^{k}$ and property of B\'{e}zier function, for $\forall t_0\in[T_k,T_{k+1}]$, it follows that 
\begin{equation}
\begin{aligned}
s(t_0) \in \mathcal{S}_k^{rec}=
\{(t_i,s_i)|\underline{p_0^{k}}+h_k\underline {p_1^{k}}\leq s_i \leq \overline {p_0^{k}},\\t_i\in[T_k,T_{k+1}]\}\subseteq \mathcal{S}_k.
\end{aligned}
\end{equation}
Further, we have $\mathcal{S}^{rec}\subseteq \mathcal{S}$. 
\end{proof}
\vspace{-0.18cm}
The problem of using rectangular corridors on S-T graph is that when the bounds of corridors meet $\underline{p_0^{k}}+h_k\underline {p_1^{k}}> \overline {p_0^{k}}$, there is no feasible solution of the optimization problem and the planner will fail. In order to avoid this case, the time intervals of $k$-th corridors should satisfy, $h_k\leq\frac{\overline{p_0^k}-\underline{p_0^k}}{\underline{p_1^k}}$.

Accordingly, \cite{ding2019safe} proposes a seed generation and cube inflation method to adjust time intervals with proper values cleverly. However, this will generate too many corridors and optimized variables in complex environments, which leads to a loss of computation efficiency. Besides, if $\exists$ $\underline{p_1^{k}}>0$ or $\overline{p_1^k}>0$, we have $\mathcal{S}^{rec}\subsetneqq \mathcal{S}$.
The safe regions are not fully utilized for optimization of the station curve.
This reflects that the constraints on control points are too tight to reduce the solution space and degrade the optimality of designed stations. 

Next, we will explore safety constraints on $c_i$ in trapezoidal corridors to fully used safe regions without loss of optimality. 
\begin{figure}[t]
\begin{center}
\subfigure[Optimization in piecewise rectangular corridors (failure).]{\label{fig20.b}
\includegraphics[width=4cm,height=2.6cm]{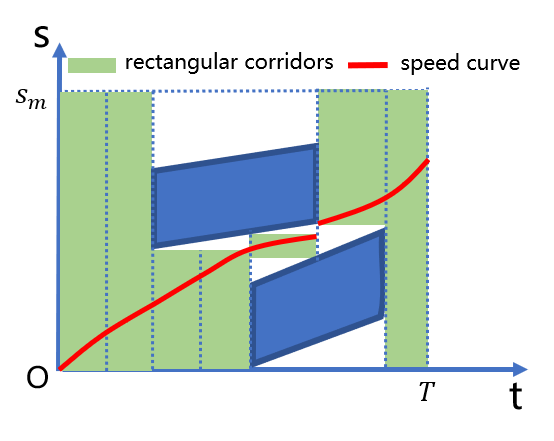}}
\subfigure[Optimization in piecewise rectangular corridors (success).]{\label{fig20.c}
\includegraphics[width=4cm,height=2.6cm]{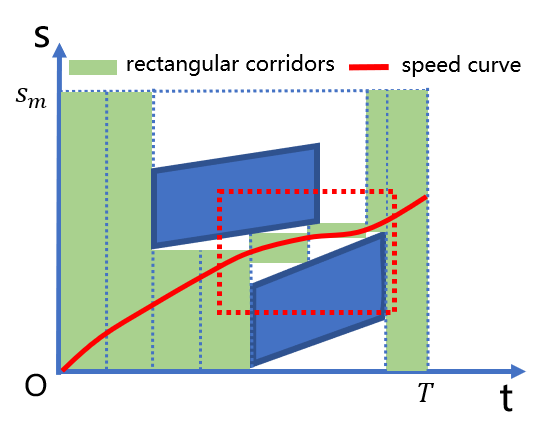}}
\subfigure[Suboptimal of optimization in rectangular corridors.]{\label{fig20.d}
\includegraphics[width=4cm,height=2.6cm]{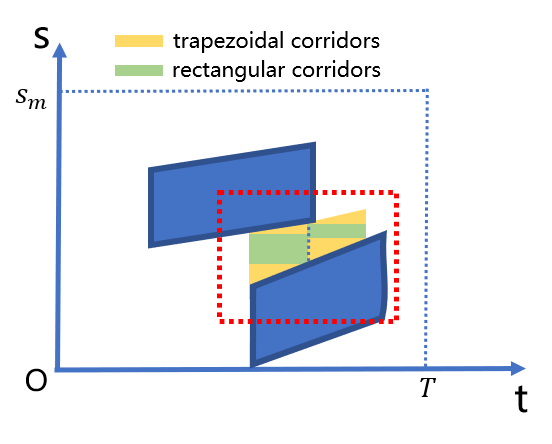}}
\subfigure[Optimization in piecewise trapezoidal corridors.]{\label{fig20.a}
\includegraphics[width=4cm,height=2.6cm]{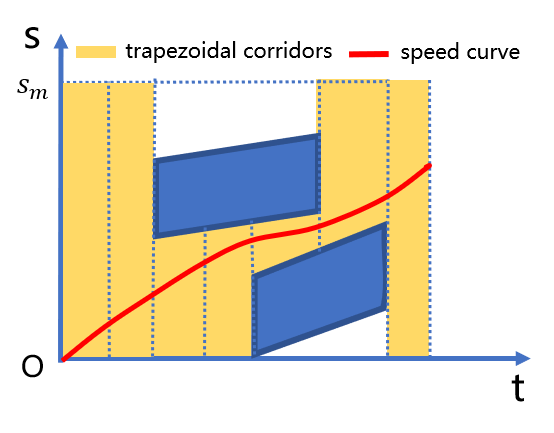}}
\caption{Optimization in trapezoidal and rectangular corridors.}
\label{fig20}
\end{center}
\vspace*{-0.7cm}
\end{figure}
\vspace{-0.7cm}
\subsection{Safety Enforcement in Trapezoidal Corridors}
Before moving to the sufficient conditions on control points $c_i$, we first give the following lemma.
\begin{lemma}\label{lem1}
Let $M\in\mathbb{R}^{(n+1)\times(n+1)}$ denote the transition matrix from the Bernstein basis $\{b_{n}^{0}(t),b_{n}^{1}(t),\dots,b_{n}^{n}(t)\}$ to the monomial basis 
$\{1,t,t^{2},\dots,t^{n}\}$. We have 
$M_{i,0}=1, \; 0 \leq M_{i,j}\leq 1, \; i=0,1,\dots,n, \; j=0,1,\dots,n$.
\end{lemma}
\begin{proof}
It follows that 
\begin{align*}
t^{i}=t^{i}(t+1-t)^{n-i}=&\sum_{j=0}^{n-i}C_{n-i}^{j}t^{n-j}(1-t)^{j}\\
=&\sum_{j=0}^{n-i}\frac{C_{n-i}^{j}}{C_{n}^{n-j}}C_{n}^{n-j}t^{n-j}(1-t)^{j}.
\end{align*}
Hence, the elements of matrix $M$ satisfy
\begin{equation}
\centering
M_{n-j,i}=\left\{
\begin{array}{cc} 
   \frac{C_{n-i}^{j}}{C_{n}^{j}},&i+j\leq n\\
   0, &i+j>n. 
\end{array}
\right.
\end{equation}
We have $M_{i,0}=1, \; 0 \leq M_{i,j}\leq 1, \; i,\;j=0,1,\dots,n$.
\end{proof}
\begin{theorem}\label{thm1}
For arbitrary piece of the trajectory, 
$\forall \underline {p_0^k},\overline {p_0^k},\underline {p_1^k},\overline {p_1^k}\in \mathbb{R}$, there exists control points $c_i^k\in \Omega_2^{k}=\{c^k|\underline {p_0^k}+h_k\underline {p_1^k}M_{i,1} \leq c_i^{k} \leq \overline {p_0^k}+h_k\overline {p_1^k}M_{i,1},i=0,1,\dots,n\}$, 
s.t. $s(t)$ is safe and $\mathcal{S}^{cor}$ is a trapezoidal corridor.
\end{theorem}
\begin{proof}
On the S-T graph for $t\in[T_k,T_{k+1}]$, it holds that 
\begin{equation}\label{safety_bound}
\underline {p_0^k}+h_k\underline {p_1^k}\frac{t-T_k}{h_k}<
\overline {p_0^k}+h_k\overline {p_1^k}\frac{t-T_k}{h_k}. 
\end{equation}
According to lemma \ref{lem1}, $M_{i,1}$ satisfies $0\leq M_{i,1}\leq 1$. 
Thus, we have $T_k\leq T_k+h_kM_{i,1}\leq T_{k+1}$ and let $t=T_k+h_kM_{i,1}$. Then we obtain 
$$\underline {p_0^k}+h_k\underline {p_1^k}M_{i,1} < \overline {p_0^k}+h_k\overline {p_1^k}M_{i,1},$$ and $\exists c_i^k,s.t.\ \underline {p_0^k}+h_k\underline {p_1^k}M_{i,1} \leq c_i^{k} \leq \overline {p_0^k}+h_k\overline {p_1^k}M_{i,1}$. 
As for the safety of $s(t)$, $\forall t_0\in [T_k,T_{k+1}]$,
\begin{align*}
s(t_0)&\leq \sum_{i=0}^{n}(\overline {p^k_0}+h_k\overline {p^k_1}M_{i,1})b_n^{i}
\left(\frac{t-T_k}{h_k}\right)\\
&\leq \overline {p^k_0}\sum_{i=0}^{n}b_n^{i}\left(\frac{t-T_k}{h_k}\right)
+h_k\overline {p_1^k}\sum_{i=0}^{n}M_{i,1}b_n^{i}\left(\frac{t-T_k}{h_k}\right)\\
& = \overline {p_0^k} + h_k\overline {p_1}\frac{t-T_k}{h_k}.
\end{align*}
Similarly, we can achieve $s(t_0)\geq \underline {p_0^k} + h_k\underline {p^k_1}
\frac{t-T_k}{h_k}$. Therefore, we obtain $s(t_0)\in\mathcal{S}_k^{tra}=\mathcal{S}_k$ 
and $s(t)\in\mathcal{S}^{tra}=\mathcal{S}$. Hence $s(t)$ is safe and the corridor is a trapezoid.
\end{proof}
The proof of Theorem \ref{thm1} utilizes the idea of transforming terms in boundaries into the Bernstein basis. The process of proof also shows that $s(t)$ is exactly constrained in the safe set $\mathcal{S}$ according to the convex property of B\'{e}zier function, not its subset, e.g. $\mathcal{S}^{rec}$.

In Theorem \ref{thm1}, conditions on $c_i$ is $\underline {p_0^k}+h_k\underline {p_1^k}M_{i,1} \leq c_i^{k} \leq \overline {p_0^k}+h_k\overline {p_1^k}M_{i,1}$. Compared to safety enforcement in rectangular corridors in proposition \ref{propo2}, we have $\underline {p_0^k}+h_k\underline {p_1^k}M_{i,1}\leq \underline {p_0^k}+h_k\underline {p_1^k}$ and $\overline {p_0^k}+h_k\overline {p_1^k}M_{i,1} \geq \overline {p_0^k}$. The advantage of this is two folds: i) By the proof of $\underline {p_0^k}+h_k\underline {p_1^k}M_{i,1} < \overline {p_0^k}+h_k\overline {p_1^k}M_{i,1}$, the lower boundaries are smaller than the higher boundaries all the time. Therefore, the planner will not be insoluble and fail with safety enforcement. ii) The constraints are relaxed and the solutions tend to be better than the ones by rectangular corridors. 

\begin{remark}
Considering the corridors with $n$th-ordered polynomials as boundaries, if there exists control points 
$c_i^k \in \Omega_3^{k}=\{c^k|\sum_{j=0}^{n}h_k^{j}\underline {p_j^k}M_{i,j}\leq c^k \leq 
\sum_{j=0}^{n}h^j_k\overline {p_j^k}M_{i,j},i=0,1,\dots,n\}$, it follows that 
$\sum_{j=0}^{n}h^j_k\underline{p^k_j}{\left(\frac{t-T_k}{h_k}\right)}^{j}\leq s(t) \leq\sum_{j=0}^{n}h^j_k\overline{p^k_j}{\left(\frac{t-T_k}{h_k}\right)}^j$ and $s(t)$ is safe.
However, the lower bound may be larger than the upper bound in this inequality and there is no solution for the optimization problem.
\end{remark}

\subsection{Trajectory Optimization Formulation}
There are several objectives for the B\'{e}zier function to minimize and 
the cost function is designed as
\begin{small}
\begin{equation}
\begin{aligned}
&J = w_1\int_{0}^{T}\left(s(t)-s^{r}(t)\right)^2\mathrm{d}t+
w_2\int_{0}^{T}\left(\dot{s}(t)-v^{r}\right)^2\mathrm{d}t\\
&+w_3\int_{0}^{T}\ddot{s}(t)^2\mathrm{d}t+
w_4\int_{0}^{T}\dddot{s}(t)^2\mathrm{d}t+w_5
\left(s(T)-s^r(T)\right)^2,
\end{aligned}
\end{equation}
\end{small}where $s^{r}(t)$ is the reference stations by DP and $v_r$ is the cruise speed. The cruise speed varies with scenarios, but is considered as a constant in one planning period. The first term is to minimize the distance between the B\'{e}zier curve and the heuristic $s-t$ profile. The second one optimizes the bias between the actual speed and reference speed. This aims to let the vehicle keep a high velocity. The third and fourth terms are to smooth the heuristic $s-t$ curve by penalizing the acceleration and jerk respectively. Besides, the end station is expected to reach the certain value $s^r(T)$ by the last term. 

The constraints include boundary constraints, continuity constraints, safety constraints and physical constraints. \\
\noindent i) Boundary Constraints:
The piecewise curve starts at fixed position, speed and acceleration and it follows that
\begin{equation}
\begin{aligned}
c_{i}^{0,l}h_k^{(1-l)}=\left.\frac{\mathrm{d}^l s(t)}{\mathrm{d}t^l}\right|_{t=0},
\ \; l=0,1,2,
\end{aligned}
\end{equation}
where $c_{i}^{k,l}$ is the control point for the $l$-th order derivative of the $k$-th B\'{e}zier curve. \\
ii) Continuity Constraints:
The piecewise curve is continuous at the connecting time points in terms of position, speed and acceleration. 
According to these conditions, we come to 
\begin{equation}
\begin{aligned}
c_n^{k,l}h_k^{(1-l)}=c_{0}^{k+1,l}h_{k+1}^{(1-l)}, \;l=0,1,2, \; k=0,1,\dots,m-1.
\end{aligned}
\end{equation}
iii) Safety Constraints:
With trapezoidal corridors as discussed in the last part, safety constraints can be represented as
\begin{equation}
\begin{aligned}
\underline {p_0^k}+h_k\underline {p_1^k}M_{i,1} \leq c_{i}^{k,0} \leq 
\overline {p_0^k}+h_k\overline {p_1^k}M_{i,1}, \; k=0,1,\dots,m. 
\end{aligned}
\end{equation}
iv) Physical Constraints:
These constraints consider real physical conditions of the vehicle and limit the velocity, acceleration and jerk. We first use the hodograph property (iv) to calculate velocity, acceleration and jerk as B\'{e}zier polynomials. Then the constraints can be given by 
\begin{equation}
\begin{aligned}
\underline {\beta^{k,1}}& \leq c_{i}^{k,l} \leq 
\overline {\beta^{k,1}},\\
\underline {\beta^l}& \leq c_{i}^{k,l} \leq 
\overline {\beta^l},\ l=2,3,
\end{aligned}
\end{equation}
where $k=0,1,\dots,m$ and it follows that $c_{i}^{k,l+1}=(n-l)(c_{i+1}^{k,l}-c_{i}^{k,l})$. 
The upper bounds $\overline {\beta^{k,1}}$ are determined by speed limits on road and centripetal acceleration constraints. Let $a_{cm}$ be the maximum acceleration permitted and $\kappa_{k}$ the maximum curvature of the path for $t\in[T_k,T_{k+1}]$ (see \cite{zhang2019optimal} for details). The lateral acceleration constraints are given by 
\begin{equation}
\begin{aligned}
 c_{i}^{k,l} \leq 
\overline {\beta^{k,1}}=
\sqrt{\frac{a_{cm}}{\kappa_{k}}}.
\end{aligned}
\end{equation}
The bounds on longitudinal accelerations and jerks are constant for different pieces of speed profiles.

Then, the trajectory optimization process can be formulated as a quadratic programming (QP) problem as 
\begin{equation}\label{eq11}
\begin{aligned}
\bm{\mathrm{P}}:\qquad\min\limits _{\bm{\mathrm{c}}}\;\; & \bm{\mathrm{c}}^{T}\bm{\mathrm{Q_{c}c}}+\bm{\mathrm{q_{c}}}^{T}\bm{\mathrm{c}}+\mathrm{const}\\
\text{s.t.}\;\; & \bm{\mathrm{A}}_{eq}\bm{\mathrm{c}}=\bm{\mathrm{b}}_{eq}\\
 & \bm{\mathrm{A}}_{ie}\bm{\mathrm{c}}\leq\bm{\mathrm{b}}_{ie}.
\end{aligned}
\end{equation}
We refer readers to the appendix for the detailed formulation process. This problem can be solved by solvers like OSQP. 
\section{EXPERIMENTS AND RESULTS ANALYSIS}
The planning approach proposed is implemented in C++11 with the OSQP solver \cite{stellato2020osqp}. All experiments are carried out on a dual-core 2.90GHz Intel i5-4210H processor. The planning horizon is 7s. In the dynamic programming procedure, the resolution of time is 1s. Hence, in each corridor, the heuristic station $s^r(t)$ is a first-order polynomial. 
\subsection{Optimality and Low Failure Rate Verification}
We conduct experiments on the S-T graph first to verify the optimality and low failure rate of the proposed approach compared to B\'{e}zier polynomial with rectangular corridors. Consider the scenario in Fig.\ref{fig1.a}. We project different stations of the vehicles onto the S-T graph. Suppose the ego vehicle is currently traveling at a speed of $v(0)$ $=$ $10.0 m/s$ and an acceleration of $a(0) = 0$. 
The parameters are set to be $w_1 = 0.1$, $w_2 = 0.1$, $w_3 = 10.0$, $w_4 = 5.0$ 
and $w_5 = 3.0$. 

As shown in Fig.\ref{fig3.a}, B\'{e}zier curves are calculated within rectangular and trapezoidal corridors with same parameters. Fig.\ref{fig3.c} illustrates corresponding speed profiles and \ref{fig3.d} shows acceleration profiles. We use longitudinal accelerations to quantify the comfort. As shown in the figure, the maximum acceleration of the curve generated by trapezoidal corridors is smaller than that by rectangular corridors. 
Table \ref{table2} gives maximal accelerations and average accelerations 
$\sqrt{\frac{1}{T}\int_{0}^{T} (\ddot{s}(t))^2\mathrm{d}t}$ in planning horizon $[0,T]$ and $T=7s$. This proves that TC has better comfort and smoothness performance than RC.

In the case shown in Fig.\ref{fig3.b}, the TC method works and the RC approach does not find a feasible solution. This is because the solution space of TC is larger than that of RC. As a result, the TC method has a lower failure rate in complex situations. 
\begin{figure}[t]
\begin{center}
\subfigure[B\'{e}zier curves with obstacles.]{\label{fig3.a}
\includegraphics[width=4cm,height=3cm]{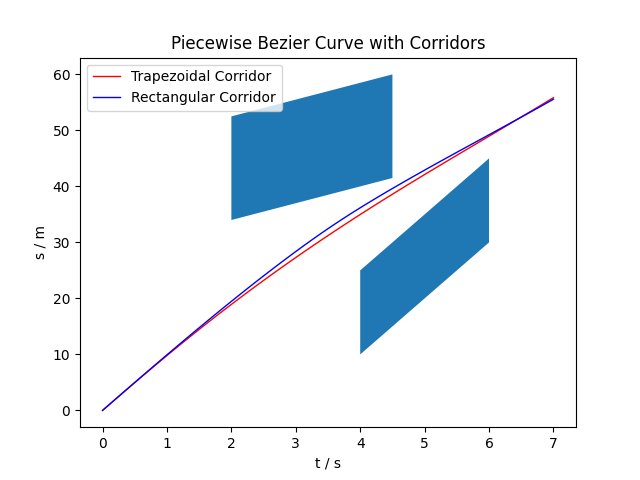}}
\subfigure[B\'{e}zier curves with obstacles.]{\label{fig3.b}
\includegraphics[width=4cm,height=3cm]{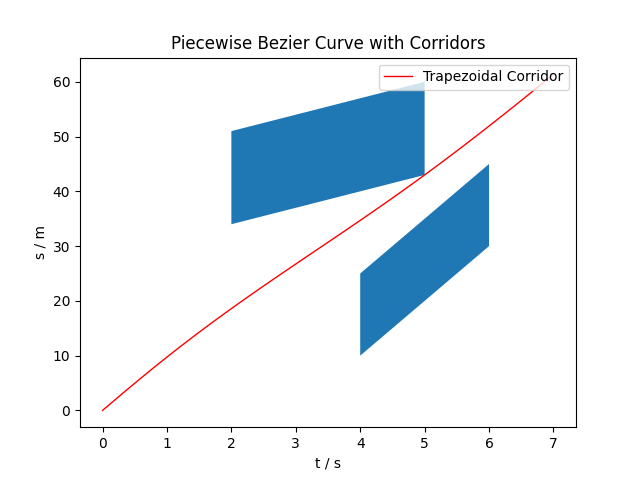}}
\subfigure[Speed profiles.]{\label{fig3.c}
\includegraphics[width=4cm,height=3cm]{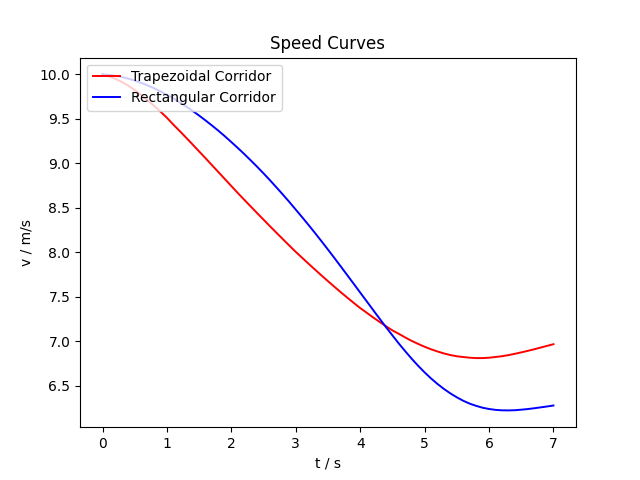}}
\subfigure[Acceleration profiles.]{\label{fig3.d}
\includegraphics[width=4cm,height=3cm]{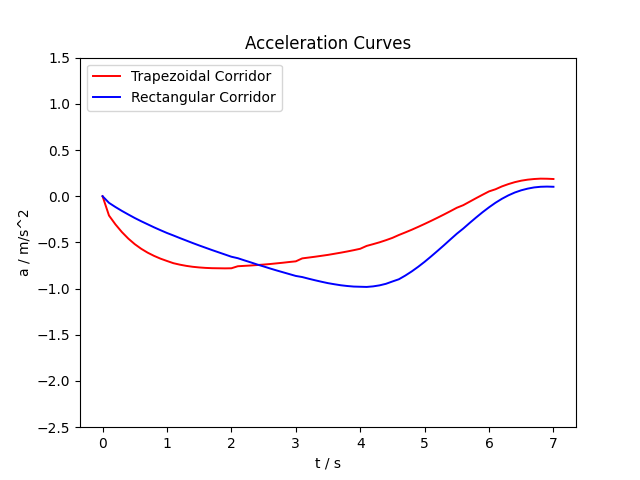}}
\caption{B\'{e}zier curves within trapezoidal and rectangular corridors.}
\label{fig3}
\end{center}
\vspace*{-0.8cm}
\end{figure}
\begin{table}[!htbp]
\vspace{-0.6cm}
\centering
\setlength{\abovecaptionskip}{6pt}%
\setlength{\belowcaptionskip}{0pt}%
\caption{Comfort of planners by different corridors.}
\setlength{\tabcolsep}{1.4mm}
\vspace{-0.21cm}
{
\begin{tabular}{c|c|c|c|c}
\label{table2}
  &  \multicolumn{2}{c|}{1} & \multicolumn{2}{c}{2} \\
\hline
  & Max. Acc. & Ave. Acc. & Max. Acc. & Ave. Acc. \\
\hline
 RC & 0.95 & 0.62 & - & -\\
\hline
 TC & \textbf{0.78} & \textbf{0.54} & \textbf{0.96}  & \textbf{0.59} \\
\hline
\end{tabular}}
\vspace{-0.7cm}
\end{table}
\subsection{Efficiency Verification}
We test our speed planning method in ROS and compare it with the start-of-art approaches to illustrate its efficiency. The perception and prediction modules are omitted by the simulation scenario and the localization of ego vehicle are obtained by the simulator. The low-level controller is implemented by Pure Pursuit algorithm. The motion planning part includes path planner implemented by the same method as \cite{zhou2020autonomous} and speed planners implemented by different approaches for tests. 

In Fig.\ref{fig4.a}, the ego vehicle (blue) begins to move forward. The global path is planned and shown as the long green line, covered by a red line representing local path planning for obstacle avoidance. The green line is also known as the guided line to generate the Frenet frame. Fig.\ref{fig4.b} - Fig.\ref{fig4.d} shows that the vehicle perceives the passing vehicle (orange) at the intersection and slows down to avoid the obstacle successfully. Fig.\ref{fig4.e} illustrates that the car passes the turn. Fig.\ref{fig4.f} shows that the vehicle reaches the destination. 

We conduct the experiment twice to test the efficiencies of different speed planners and the Table \ref{table1} contrasts the results. Our method outperforms EM planner \cite{fan2018baidu} and Public Road (PR) Planner \cite{zhang2019optimal} with 5.70 ms and 5.42 ms. 
The standard deviations and worst time values are also better than others.

\begin{figure}[t]
\begin{center}
\subfigure[Beginning to move.]{\label{fig4.a}
\includegraphics[width=4cm,height=2.5cm]{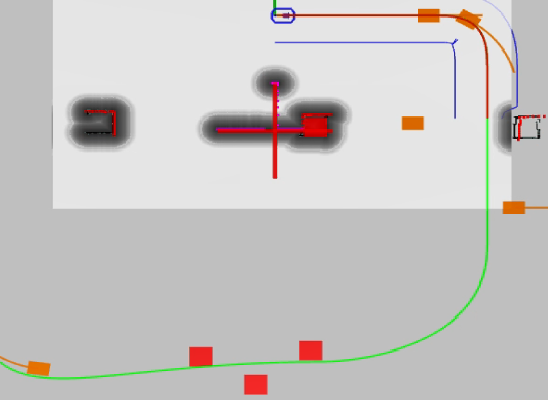}}
\subfigure[Avoiding the dynamic obstacle.]{\label{fig4.b}
\includegraphics[width=4cm,height=2.5cm]{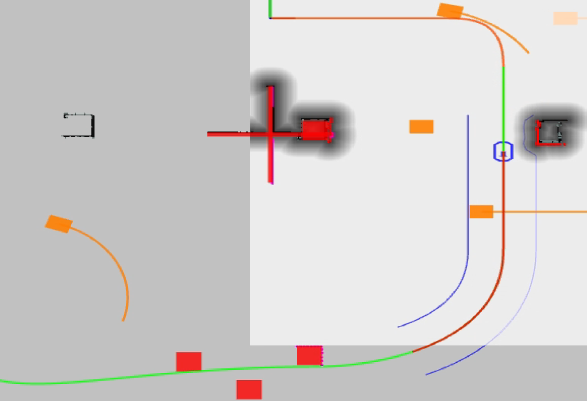}}
\subfigure[Avoiding the dynamic obstacle.]{\label{fig4.c}
\includegraphics[width=4cm,height=2.5cm]{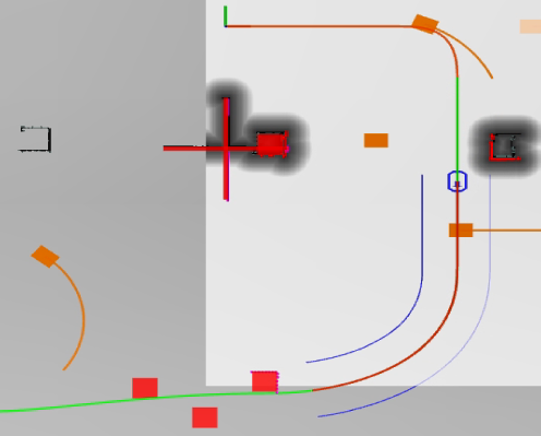}}
\subfigure[Avoiding the dynamic obstacle.]{\label{fig4.d}
\includegraphics[width=4cm,height=2.5cm]{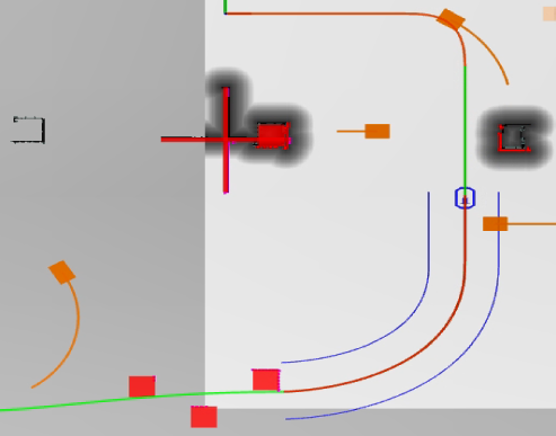}}
\subfigure[Passing the turn successfully.]{\label{fig4.e}
\includegraphics[width=4cm,height=2.5cm]{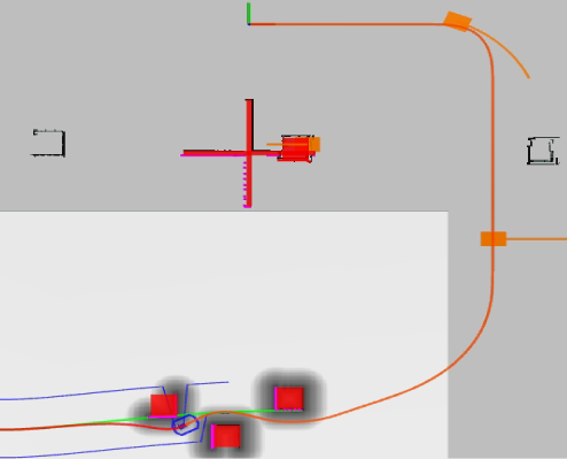}}
\subfigure[Reaching the destination.]{\label{fig4.f}
\includegraphics[width=4cm,height=2.5cm]{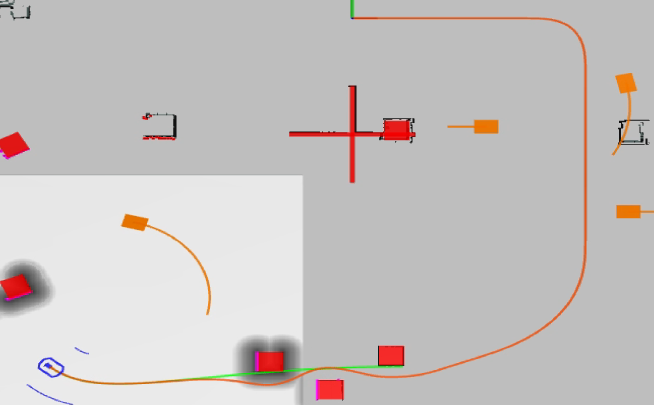}}
\caption{Autonomous navigation of the ego vehicle.}
\label{fig4}
\end{center}
\vspace*{-10pt}
\end{figure}

\begin{table}[!htbp]
\centering
\setlength{\abovecaptionskip}{6pt}%
\setlength{\belowcaptionskip}{0pt}%
\caption{Time consumptions of different velocity planners.}
\setlength{\tabcolsep}{1.4mm}
{
\begin{tabular}{c|c|c|c|c|c}
\label{table1}
  & Velocity Planner & Ave./ms & Std./ms & Worst/ms & Planning times \\
\hline
\multirow{3}*{1} & EM  & 16.81 & 7.77 & 39.57& 379\\
\cline{2-6}
 & PR &7.91&4.04 & 28.71& 429\\
\cline{2-6}
 & Ours & \textbf{5.70} & \textbf{2.72} & \textbf{16.46} & 403\\
\hline
\multirow{3}*{2} & EM & 15.52 & 6.71 & 38.37& 421\\
\cline{2-6}
 & PR &7.89 &3.69 & 35.47& 548\\
\cline{2-6}
 & Ours & \textbf{5.42} & \textbf{2.45} & \textbf{13.61} & 470\\
 \hline
\end{tabular}}
\vspace{-0.65cm}
\end{table}
\section{Conclusion}
In this paper, we investigate speed planning for autonomous vehicles. 
We use dynamic programming to search coarse stations first. Then an approach to generating trapezoidal corridors is presented. We also propose the sufficient conditions on control points of B\'{e}zier polynomials to guarantee the safety theoretically. Compared to previous safety enforcement in rectangular corridors, our method is proved to be more relaxed and solvable when incorporated into the B\'{e}zier polynomials optimization. After that, we formulate the velocity optimization as a QP problem. Simulations show that our approach is better than the safety enforcement in rectangular corridors in terms of optimality and low failure rates. The proposed way is also faster than the start-of-art methods, e.g. EM Planner and Public Road Planner.
\vspace{-0.3cm}
\section*{appendix}
\subsection{Proof of Proposition \ref{pro1}}
\begin{proof}
Since $c_i \leq \overline {p_0}$, it follows that 
$$B(t)\leq \overline {p_0}\sum_{i=0}^{n}b_{n}^{i}(t)=\overline {p_0}(t+1-t)^{n}=\overline {p_0}.$$
Similarly, we have $B(t)\geq \underline {p_0}$. Hence, it holds that $\underline {p_0} \leq B(t) \leq \overline {p_0},t\in[0,1]$.
\end{proof}
\subsection{QP Formulation}
This part illustrates how to formulate the B\'{e}zier polynomial optimization as a QP problem. First, we express the B\'{e}zier curve as a polynomial \vspace{-0.25cm}
\begin{footnotesize}\begin{equation}
\begin{aligned}
s_k(t)=&h_k\sum_{i=0}^{n}c_i^kb_n^{i}\left (\frac{t-T_k}{h_k}\right)\\
=&h_k\sum_{i=0}^{n}p_i^k\left (\frac{t-T_k}{h_k}\right)^{i}
=h_kf_k\left (\frac{t-T_k}{h_k}\right),
\end{aligned}
\vspace{-0.2cm}
\end{equation}\end{footnotesize}

\noindent where $f_k(t)=\sum_{i=0}^{n}p_i^kt^i$, \; $k=0,1,\dots,m$ is a polynomial curve.
Let $M\in\mathbb{R}^{(n+1)\times(n+1)}$ denote the transition matrix from the Bernstein basis $\{b_{n}^{0}(t),b_{n}^{1}(t),\dots,b_{n}^{n}(t)\}$ to the monomial basis $\{1,t,t^{2},\dots,t^{n}\}$. Then, we have 
$\bm{\mathrm{c^k}}=M\bm{\mathrm{p^k}}$
with $\bm{\mathrm{c^k}}=[c_0^k,\dots,c_n^k]^{T}$ and 
$\bm{\mathrm{p^k}}=[p_0^k,\dots,p_n^k]^{T}$. According to lemma \ref{lem1}, it holds that $|M|>0$ and $M$ is invertible. Hence, if the objective function can be written as
\vspace{-0.25cm}
\begin{footnotesize}
\begin{equation}\label{QP_FORM}
\begin{aligned}
J=\sum_{k=0}^{m}\left[(\bm{\mathrm{p^k}})^{T}Q^k{\bm{\mathrm{p^k}}}+\bm{\mathrm{q^k}}\bm{\mathrm{p^k}}\right]+\mathrm{const}\geq0,
\end{aligned}
\vspace{-0.25cm}
\end{equation}
\end{footnotesize}
where $Q^k$ is positive definite and known, then we have
\begin{footnotesize}
\begin{equation}
\begin{aligned}
 J=&\left[\begin{array}{c}
 \bm{\mathrm{c^0}}\\\vdots \\ \bm{\mathrm{c^m}}
 \end{array}\right]^T
 \left[\begin{array}{ccc}
 {(M^{-1})}^{T}Q^0M^{-1} & & \bm{\mathrm{0}}\\
   &  \ddots  &\\
  \bm{\mathrm{0}} &  & {(M^{-1})}^{T}Q^mM^{-1}
 \end{array}\right]
 \left[\begin{array}{c}
 \bm{\mathrm{c^0}}\\\vdots \\ \bm{\mathrm{c^m}}
 \end{array}\right]\\
&+\left[\begin{array}{c}
{\bm{\mathrm{q^0}}} \\ \vdots \\ {\bm{\mathrm{q^m}}}
\end{array}\right]^T
\left[\begin{array}{ccc}
 M^{-1} & & \bm{\mathrm{0}}\\
   &  \ddots  &\\
  \bm{\mathrm{0}} &  & M^{-1}
 \end{array}\right]
 \left[\begin{array}{c}
 \bm{\mathrm{c^0}}\\\vdots \\ \bm{\mathrm{c^m}}
 \end{array}\right]+\mathrm{const}\\
 =&\bm{\mathrm{c}}^{T}Q_c\bm{\mathrm{c}}+\bm{\mathrm{q_{c}}}^{T}\bm{\mathrm{c}}+\mathrm{const}
 \geq0.
\end{aligned}
\end{equation}
\end{footnotesize}$Q_c$ is also a positive-definite matrix. 
Since the constraints are all linear with $\bm{\mathrm{c}}$, the optimization problem is a QP problem. 

Next we will illustrate that equation \eqref{QP_FORM} holds and how to calculate $Q_k$ and $\bm{\mathrm{q^k}}$. We first calculate some terms to achieve the cost function $J$. To begin with, it holds that \vspace{-4pt}
\begin{footnotesize}
\begin{equation}
\begin{aligned}
&\int_{T_k}^{T_{k+1}}\left(\frac{\mathrm{d}^ls(\tau)}{\mathrm{d}\tau^l}\right)^2\mathrm{d}\tau
=\int_{0}^{h_k}\left(\frac{\mathrm{d}^ls(\tau+T_k)}{\mathrm{d}\tau^l}\right)^2\mathrm{d}\tau\\
&=\int_{0}^{h_k}\left(\frac{\mathrm{d}^ls(\tau+T_k)}{\mathrm{d}t^l}
\left(\frac{\mathrm{d}t}{\mathrm{d}\tau}\right)^l\right)^2\mathrm{d}\tau
=\frac{1}{h_k^{2l-3}}\int_{0}^{1}\left(\frac{\mathrm{d}^lf_k(t)}{\mathrm{d}t^l}\right)^2\mathrm{d}t.
\end{aligned}
\end{equation}\end{footnotesize} \vspace{-4pt}
As for $\int_{0}^{1}\left(\frac{\mathrm{d}^lf_k(t)}{\mathrm{d}t^l}\right)^2\mathrm{d}t$, it follows that 
\begin{footnotesize}
\begin{equation}\label{de_integral}
\begin{aligned}
&\int_{0}^{1}\left(\frac{\mathrm{d}^lf_k(t)}{\mathrm{d}t^l}\right)^2\mathrm{d}t
=\int_{0}^{1}\sum_{i\geq l,j\geq l}p_i^kp_j^kt^{i+j-2l}\mathrm{d}t\\
=&\sum_{i\geq l,j\geq l}\frac{i(i-1)\cdots(i-l)j(j-1)\cdots(j-l)}{i+j+1-2l}p_i^kp_j^k.
\end{aligned}
\end{equation}\end{footnotesize}We also have $\int_{0}^{1}tf_k(t)\mathrm{d}t=\sum_{i}\frac{1}{i+2}p_i^k$,$
\int_{0}^{1}f_k(t)\mathrm{d}t=\sum_{i}\frac{1}{i+1}p_i^k.$
Suppose $J=\sum_{i=1}^{5}w_iJ_i$, the terms of $J$ satify \vspace{-0.1cm}
\begin{footnotesize}
\begin{equation}\label{objective_function1}
\begin{aligned}
\quad J_1=&\sum_{k=0}^{m}\int_{T_k}^{T_{k+1}}\left(s_k(t)-a_k(t-T_k)-b_k\right)^2\mathrm{d}t\\
   =&\sum_{k=0}^{m}\int_{T_k}^{T_{k+1}}\left[s_k(t)^2
   -2\left(a_k(t-T_k)+b_k\right)s_k(t)\right]\mathrm{d}t
   +\mathrm{const}\\
   =&\sum_{k=0}^{m}h_k^3\int_{0}^{1}f_k(t)^2\mathrm{d}t-
   2h_k^3a_k\int_{0}^{1}tf_k(t)\mathrm{d}t-2h_k^2b_k\int_{0}^{1}f_k(t)\mathrm{d}t\\&+\mathrm{const}\\
\end{aligned}
\end{equation}\end{footnotesize}\vspace{-0.3cm}
\begin{footnotesize}
\begin{equation}\label{objective_function2}
\begin{aligned}
\quad J_2=&\sum_{k=0}^{m}\int_{T_k}^{T_{k+1}}{\dot{s_k}(t)}^2\mathrm{d}t
   -2v_r\int_{0}^{T}\dot{s}(t)\mathrm{d}t+\mathrm{const}\\
   =&\sum_{k=0}^{m}h_k\int_0^1\dot{f}_k(t)^2\mathrm{d}t-2v_rs(T)+\mathrm{const}
\end{aligned}
\end{equation}\end{footnotesize}
\begin{footnotesize}
\begin{equation}\label{objective_function3}
\begin{aligned}
\quad J_3=\sum_{k=0}^{m}\frac{1}{h_k}\int_0^1\ddot{f}_k(t)^2\mathrm{d}t,
\quad J_4=\sum_{k=0}^{m}\frac{1}{h^3_k}\int_0^1\dddot{f}_k(t)^2\mathrm{d}t
\end{aligned}
\end{equation}\end{footnotesize}
\begin{footnotesize}
\begin{equation}\label{objective_function5}\begin{aligned}
\quad J_5=&\left(s(T)-s^r(T)\right)^2=s(T)^2-2s^r(T)s(T)+\mathrm{const}.
\end{aligned}
\end{equation}\end{footnotesize}
Then we can come to the equation \eqref{QP_FORM} by replacing integral terms and using $J=\sum_{i=1}^{5}w_iJ_i$.
\vspace{-6pt}

\end{document}